\newif\ifcolt
\coltfalse
\ifcolt
\documentclass[anon]{colt2013}
\newcommand{\numberofauthors}[1]{} \newcommand{\affaddr}[1]{#1} \newcommand{\email}[1]{\href{mailto:#1}{\texttt{#1}}} \newcommand{\myand}{\and} \newcommand{\alignauthor}{} \newcommand{\myproof}{Proof } 
\else
\documentclass[12pt]{article}\newcommand{\numberofauthors}[1]{} \newcommand{\affaddr}[1]{#1} \usepackage{amsthm} \newcommand{\email}[1]{\href{mailto:#1}{\texttt{#1}}} \newcommand{\myand}{\and} \newcommand{\alignauthor}{} \newcommand{\myproof}{Proof } \usepackage{fullpage}
\fi

\usepackage{amsmath, amssymb}
\usepackage{bbm}
\usepackage{latexsym}
\usepackage{color}

\usepackage{algorithm}
\usepackage{algpseudocode}

\usepackage{float}

\floatstyle{ruled}
\newfloat{subroutine}{htbp}{loa}
\floatname{subroutine}{Subroutine}
\algdef{SE}[myFOR]{myFor}{myEndFor}[1]{\algorithmicfor\ #1\ \algorithmicdo}{\algorithmicend\ \algorithmicfor}%

\usepackage{mathtools}
\usepackage{amsfonts}
\usepackage{hyperref}
\newcommand{\optionalbreak}{}
\ifcolt
 \coltauthor{\Name{Joseph Anderson} \Email{andejose@cse.ohio-state.edu}\\
 \addr Computer Science and Engineering \\ Ohio State University
 \AND
 \Name{Navin Goyal} \Email{navingo@microsoft.com}\\
 \addr Microsoft Research India
 \AND
 \Name{Luis Rademacher} \Email{lrademac@cse.ohio-state.edu}\\
 \addr Computer Science and Engineering \\ Ohio State University
 }
\newcommand{\mykeywords}[1]{\begin{keywords}#1\end{keywords}}

\else
\numberofauthors{2}
\author{
\alignauthor Joseph Anderson \\
\affaddr{Computer Science and Engineering} \\
\affaddr{Ohio State University} \\
\email{andejose@cse.ohio-state.edu}
\myand Navin Goyal\\
\affaddr{Microsoft Research India} \\
\email{navingo@microsoft.com}
\myand Luis Rademacher \\
\affaddr{Computer Science and Engineering} \\
\affaddr{Ohio State University} \\
\email{lrademac@cse.ohio-state.edu}}

\newtheorem{theorem}{Theorem}
\newtheorem{lemma}[theorem]{Lemma}

\newtheorem{definition}[theorem]{Definition}

\providecommand{\acks}[1]{\section*{Acknowledgments}{#1}}
\newcommand{\mykeywords}[1]{}
\fi

\newenvironment{proofidea}{\noindent{\textit{Proof idea.}}}{\hfill$\square$\medskip}

\newcommand{\card}[1]{\lvert#1\rvert}
\renewcommand{\dim}{{n}}

\newcommand{\RR}{\ensuremath{\mathbb{R}}}
\newcommand{\ones}{\ensuremath{\mathbbm{1}}}
\newcommand{\expectation}{\operatorname{\mathbb{E}}}
\newcommand{\e}{\expectation}
\newcommand{\conv}{\operatorname{conv}}
\newcommand{\suchthat}{\mathrel{:}}
\newcommand{\norm}[1]{{\lVert#1\rVert}}
\newcommand{\diag}{\operatorname{diag}}
\newcommand{\eps}{\epsilon}

\newcommand{\poly}{\operatorname{poly}}
\newcommand{\polylog}{\operatorname{polylog}}
\newcommand{\abs}[1]{\lvert#1\rvert}
\newcommand{\vol}{\operatorname{vol}}
\newcommand{\expdist}[1]{\mathrm{Exp({#1})}}
\newcommand{\noname}[1]{}

\newcommand{\dimension}{n}

\newtheorem{problem}{Problem}

\title{Efficient Learning of Simplices}

\date{}

\begin{document}
\maketitle
\begin{abstract}
We show an efficient algorithm for the following problem: Given uniformly random points from an arbitrary $\dim$-dimen\-sional simplex, estimate the simplex. The size of the sample and the number of arithmetic operations of our algorithm are polynomial in $\dim$.
This answers a question of Frieze, Jerrum and Kannan~\cite{FJK}.
Our result can also be interpreted as efficiently learning the intersection of $\dim+1$ half-spaces in $\RR^\dim$ in the model where the intersection is bounded and we are given polynomially many uniform samples from it.
Our proof uses the local search technique from Independent Component Analysis (ICA), also used by \cite{FJK}. Unlike these previous
algorithms, which were based on analyzing the fourth moment, ours is based on the third moment.

We also show a direct connection between the problem of learning a simplex and ICA: a simple randomized reduction to ICA from the problem of learning a simplex. The connection is based on a known representation of the uniform measure on a simplex. Similar representations lead to a reduction from the problem of learning an affine transformation of an $n$-dimensional $\ell_p$ ball to ICA. 
\end{abstract}

\mykeywords{independent component analysis, randomized reductions, learning convex bodies, method of moments}

\section{Introduction}
We are given uniformly random samples from an unknown convex body in $\RR^\dim$, how many samples are needed to approximately reconstruct
the body? It seems intuitively clear, at least for $\dim = 2, 3$, that
if we are given sufficiently many such samples then we can reconstruct (or learn) the body with very little error. For general $n$, it is known to require $2^{\Omega(\sqrt{\dim})}$
samples~\cite{GR09} (see also \cite{KOR} for a similar lower bound in a different but related model of learning).
This is an information-theoretic lower bound and no computational considerations are involved. As mentioned
in \cite{GR09}, it turns out that if the body has few facets (e.g. polynomial in $\dim$), then polynomial in $\dim$ samples are
sufficient for approximate reconstruction.
This is an information-theoretic upper bound and no efficient algorithms (i.e., with running time poly$(n)$) are known.
(We remark that to our knowledge the same
situation holds for polytopes with poly$(n)$ vertices.)  In this paper we study the
reconstruction problem for the special case when the input bodies are restricted to be (full-dimensional) simplices. We show that
in this case one can in
fact learn the body efficiently. More precisely, the algorithm  knows that the input body is a simplex but only up to an affine transformation, and the problem is to recover this affine transformation.
This answers a question of \noname{Frieze~et~al.~}\cite[Section 6]{FJK}.

The problem of learning a simplex is also closely related to the well-studied problem of learning intersections of half-spaces.
Suppose that
the intersection of $\dim+1$ half-spaces in $\RR^\dim$ is bounded, and we are given poly$(\dim)$ uniformly random samples from it. Then our
learning simplices result directly implies that we can learn the $\dim+1$ half-spaces. This also has the advantage of being a proper
learning algorithm, meaning that the output of the algorithm is a set of $\dim+1$ half-spaces, unlike many of the previous algorithms.

\paragraph{Previous work.}

Perhaps the first approach to learning simplices that comes to mind is to find a minimum volume simplex containing the samples. This can be shown
to be a good approximation to the original simplex. (Such minimum volume estimators have been studied in machine learning literature, see e.g. \cite{ScholkopfPSSW01} for the problem of estimating the support of a probability distribution. We are not aware of any
technique that applies to our situation and provides theoretical guarantees.)
However, the problem of finding a minimum volume simplex is in general
NP-hard~\cite{Packer}. This hardness is not directly applicable for our problem because our input is a random sample and not a general
point set. Nevertheless, we do not have an algorithm for directly finding a minimum volume simplex; instead we use ideas similar to those used in Independent Component Analysis (ICA).
ICA studies the following problem:
Given a sample from an affine transformation of a random vector with independently distributed coordinates, recover the affine transformation (up to some unavoidable ambiguities).
\noname{Frieze~et~al.~}\cite{FJK} gave an efficient algorithm for this problem
(with some restrictions on the allowed distributions, but also with some weaker requirements than full independence)
along with most of the details of a rigorous analysis (a complete analysis of a special case can be found in \noname{Arora et al.~}\cite{Arora2012}; see also \noname{Vempala and Xiao~}\cite{VempalaX11} for a generalization of ICA to subspaces along with a rigorous analysis). 
The problem of learning parallelepipeds from uniformly random samples is a special case of this problem. 
\cite{FJK} asked if one could learn other convex bodies, and in particular simplices, efficiently from uniformly random samples.
\noname{Nguyen and Regev~}\cite{NguyenR09} gave a simpler and rigorous algorithm and analysis for the case of learning parallelepipeds with similarities to the popular FastICA algorithm of \noname{Hyv\"{a}rinen~}\cite{Hyvarinen99}.
The algorithm in \cite{NguyenR09} is a first order algorithm unlike Frieze~et~al.'s second order algorithm.

The algorithms in both \cite{FJK, NguyenR09} make use of the fourth moment function of the probability distribution. Briefly, the fourth moment in direction $u \in \RR^n$ is $\mathbb{E}(u\cdot X)^4$, where $X \in \RR^n$ is
the random variable distributed according to the input distribution. The moment function can be estimated from
the samples. The independent components of the distribution correspond to local maxima or minima of the
moment function, and can be approximately found by finding the local maxima/minima of the moment function
estimated from the sample.


More information on ICA including historical remarks can be found in \cite{ICAbook, BlindSS}.
Ideas similar to ICA have been used in
statistics in the context of projection pursuit since the mid-seventies.
It is not clear how to apply ICA to the simplex learning problem directly as there is no clear independence among the components. Let us
note that \noname{Frieze et al.~}\cite{FJK} allow certain kinds of dependencies among the components, however this does not appear to be useful
for learning simplices.

Learning intersections of half-spaces is a well-studied problem in learning theory. The problem
of PAC-learning intersections of even two half-spaces is open, and there is evidence that it is
hard at least for sufficiently large number of half-spaces: E.g., \noname{Klivans and Sherstov~}\cite{KlivansS09} prove
that learning intersections of $n^\epsilon$ half-spaces in $\RR^\dim$ (for constant $\epsilon>0$) is
hard under standard cryptographic assumptions (PAC-learning is possible, however, if one also
has access to a membership oracle in addition to random samples \cite{KwekP98}). Because of this, much effort has been expended
on learning when the distribution of random samples is some simple distribution,
see e.g. \cite{KlivansS07, Vempala10, VempalaFocs10} and references therein. This line of work makes substantial progress
towards the goal of learning intersections of $k$ half-spaces efficiently, however it falls short of
being able to do this in time polynomial in $k$ and $n$; in particular, these algorithms do not seem to be
able to learn simplices.  The distribution of samples in
these works is either the Gaussian distribution or the uniform distribution over a ball.
\noname{Frieze et al.~}\cite{FJK} and \noname{Goyal and Rademacher~}\cite{GR09} consider the uniform distribution over
the intersection. Note that this requires that the intersection be bounded. Note also that one
only gets positive samples in this case unlike other work on learning intersections of
half-spaces. The problem of learning convex bodies can also be thought of as learning a distribution or density estimation
problem for a special class of distributions.

\noname{Gravin et al.~}\cite{Gravin12} show how to reconstruct a polytope with $N$ vertices in $\RR^n$, given its first $O(nN)$ moments
in $(n+1)$ random directions. In our setting, where we have access to only a polynomial number of random samples, it's not clear
how to compute moments of such high orders to the accuracy required for the algorithm of \cite{Gravin12} even for simplices.

A recent and parallel work of \noname{Anandkumar et al.~}\cite{AnandTensor} is closely related to ours. They show that tensor decomposition
methods can be applied to low-order moments of various latent variable models to estimate their parameters. The latent variable 
models considered by them include Gaussian mixture models, hidden Markov models and latent Dirichlet allocations. The tensor 
methods used by them and the local optima technique we use seem closely related. One could view our work, as well as theirs, as 
showing that the method of moments along with existing algorithmic techniques can be applied to certain unsupervised learning problems.

\paragraph{Our results}

For clarity of the presentation, we use the following machine model for the running time: a random access machine that allows the following exact arithmetic operations over real numbers in constant time: addition, subtraction, multiplication, division and square root.

The estimation error is measured using total variation distance, denoted $d_{TV}$ (see Section \ref{sec:preliminaries}).

\begin{theorem} \label{thm:main}
There is an algorithm (Algorithm~\ref{alg:simplex} below) such that given access to random samples from a simplex
$S_{INPUT} \subseteq \RR^\dim $, with probability at least $1-\delta$ over the sample and the randomness of the algorithm, it outputs $n+1$ vectors that are the vertices of a simplex $S$ so that $d_{TV}(S, S_{INPUT}) \leq \eps$.
The algorithm runs in time polynomial in $\dim$, $1/\eps$ and $1/\delta$.
\end{theorem}

As mentioned earlier, our algorithm uses ideas from ICA. 
Our algorithm uses the third moment instead of the fourth moment used in certain versions of ICA. The third moment is not useful for learning symmetric bodies such as the cube as it is identically $0$. 
It is however useful for learning a simplex where it provides useful information, and is easier to handle than the fourth moment.
One of the main contributions of our work is the understanding of the third moment of a simplex and the structure of local maxima. This is more involved than in previous work as the simplex has no obvious independence structure, and the moment polynomial one gets has no obvious structure unlike for ICA.

The probability of success of the algorithm can be ``boosted'' so that the dependence of the running time on $\delta$ is only linear in $\log(1/\delta)$ as follows: The following discussion uses the space of simplices with total variation distance as the underlying metric space. Let $\eps$ be the target distance. Take an algorithm that succeeds with probability $5/6$ and error parameter $\eps'$ to be fixed later (such as Algorithm \ref{alg:simplex} with $\delta = 1/6$). Run the algorithm $t = O(\log 1/\delta)$ times to get $t$ simplices. By a Chernoff-type argument, at least $2t/3$ simplices are within $\eps'$ of the input simplex with probability at least $1- \delta/2$.

By sampling, we can estimate the distances between all pairs of simplices with additive error less than $\eps'/10$ in time polynomial
in $t, 1/\epsilon'$ and $\log{1/\delta}$ so that all estimates are correct with probability at least $1-\delta/2$.
For every output simplex, compute the number of output simplices within estimated distance $(2+1/10)\eps'$. With probability at least
$1-\delta$ both of the desirable events happen, and then necessarily there is at least one output simplex, call it $S$, that has $2t/3$ output simplices within estimated distance $(2+1/10)\eps'$. Any such $S$ must be within $(3+2/10) \eps'$ of the input simplex. Thus, set $\eps'=\eps/(3+2/10)$.

While our algorithm for learning simplices uses techniques for ICA, we have to do 
substantial work to make those techniques work for the simplex problem. 
We also show a more direct connection between the problem of learning a simplex and ICA: a randomized reduction from the problem of learning a simplex to ICA. The connection is based on a known representation of the uniform measure on a simplex as a normalization of a vector having independent coordinates. Similar representations are known for the uniform measure in an $n$-dimensional $\ell_p$ ball (denoted $\ell_p^n$) \cite{barthe2005probabilistic} and the \emph{cone measure} on the boundary of an $\ell_p^n$ ball \cite{schechtman1990volume, rachev1991approximate, MR1396997} (see Section \ref{sec:preliminaries}  for the definition of the cone measure). 
These representations lead to a reduction from the problem of learning an affine transformation of an $\ell_p^n$ ball to ICA. These reductions show connections between estimation problems with no obvious independence structure and ICA. 
They also make possible the use of any off-the-shelf implementation of ICA. 
However, the results here do not supersede our result for learning simplices because to our knowledge no rigorous analysis is available for 
the ICA problem when the distributions are the ones in the above reductions. 

\paragraph{Idea of the algorithm.}

 The new idea for the algorithm is that after putting
the samples in a suitable position (see below), the third moment of the sample can be used to recover the simplex using a simple 
FastICA-like algorithm. We outline our algorithm next.

As any full-dimensional simplex can be mapped to any other full-dimensional simplex by an invertible affine transformation, it is enough to determine the translation and linear transformation that would take the given simplex to some canonical simplex.
As is well-known for ICA-like problems (see, e.g., \cite{FJK}), 
this transformation can be determined \emph{up to a rotation} from the mean and the covariance matrix of the uniform distribution on the given simplex.
The mean and the covariance matrix can be estimated efficiently from a sample.
A convenient choice of an $n$-dimensional simplex is the convex hull of the canonical vectors in $\RR^{\dim+1}$.
We denote this simplex $\Delta_n$ and call it the \emph{standard simplex}.
So, the algorithm begins by picking an arbitrary invertible affine transformation $T$ that maps $\RR^\dim$ onto the hyperplane $\{x \in \RR^{\dim+1} \suchthat \ones \cdot x = 1 \}$. We use a $T$ so that $T^{-1}(\Delta_n)$ is an isotropic\footnote{See Section \ref{sec:preliminaries}.} simplex. In this case, the algorithm brings the sample set into isotropic position and embeds it in $\RR^{\dim+1}$ using $T$.
After applying these transformations we may assume (at the cost of small errors in the final result) that our sample set is obtained by sampling from an unknown rotation of the standard simplex
that leaves the all-ones vector (denoted $\ones$ from now on) invariant (thus this rotation keeps the center of mass of the standard simplex fixed), and the problem is to recover this rotation.


To find the rotation, the algorithm will find the vertices of the rotated simplex approximately. This can be done efficiently because of the following characterization of the vertices: Project the vertices of the simplex onto the hyperplane through the origin orthogonal to $\ones$ and normalize the resulting vectors. Let $V$ denote this set of $n+1$ points. Consider the problem of maximizing the third moment of the uniform distribution in the simplex along unit vectors orthogonal to $\ones$. Then $V$ is the complete set of local maxima and the complete set of global maxima (Theorem~\ref{thm:maxima}). A fixed point-like iteration (inspired by the analysis of FastICA~\cite{Hyvarinen99} and of gradient descent in \cite{NguyenR09}) starting from a random point in the unit sphere finds a local maximum efficiently with high probability. By the analysis of the coupon collector's problem, $O(\dim \log \dim)$ repetitions are highly likely to find all local maxima.

\paragraph{Idea of the analysis.}

In the analysis, we first argue that after putting the sample in isotropic position and mapping it through $T$, it is enough to analyze the algorithm in the case where the sample comes from a simplex $S$ that is close to a simplex $S'$ that is the result of applying a rotation leaving $\ones$ invariant to the standard simplex. The closeness here depends on the accuracy of the sample covariance and mean as an estimate of the input simplex's covariance matrix and mean. A sample of size $O(n)$ guarantees
(\cite[Theorem 4.1]{MR2601042}, \cite[Corollary 1.2]{1106.2775}) that the covariance and mean are close enough so that the uniform distributions on $S$ and $S'$ are close in total variation. We show that the subroutine that finds the vertices (Subroutine \ref{alg:vertex}), succeeds with some probability when given a sample from $S'$. By definition of total variation distance, Subroutine \ref{alg:vertex} succeeds with almost as large probability when given a sample from $S$ (an argument already used in \cite{NguyenR09}). As an additional simplifying assumption, it is enough to analyze the algorithm (Algorithm \ref{alg:simplex}) in the case where the input is isotropic, as the output distribution of the algorithm is equivariant with respect to affine invertible transformations as a function of the input distribution.

\paragraph{Organization of the paper.} 
Starting with some preliminaries in Sec.~\ref{sec:preliminaries}, we state some
results on the third moment of simplices in Sec.~\ref{sec:moment}. 
In Sec.~\ref{sec:standard} we give an algorithm that
estimates individual vertices of simplices in a special position; using this algorithm as a subroutine in Sec.~\ref{sec:general} we give
the algorithm for the general case. 
Sec.~\ref{sec:maxima} characterizes the set of local maxima of the third moment. 
Sec.~\ref{sec:prob} gives the probabilistic results underlying the reductions from learning simplices and $\ell_p^n$ balls to ICA.
Sec.~\ref{sec:reduction} explains those reductions.

\section{Preliminaries} \label{sec:preliminaries}
An $\dim$-simplex is the convex hull of $\dim+1$ points in $\RR^{\dim}$ that do not lie on an $(\dim-1)$-dimensional affine hyperplane.
It will be convenient to work with the standard $\dim$-simplex
$\Delta^{\dim}$ living in $\RR^{\dim+1}$ defined as the convex hull of the $\dim+1$ canonical unit
vectors $e_1, \ldots, e_{\dim+1}$; that is
\begin{align*}
\Delta^{\dim} = \{(x_0, \ldots, x_{\dim}) \in \RR^{\dim+1} &\suchthat x_0 + \dotsb + x_{\dim} = 1 \optionalbreak 
\text{ and } x_i \geq 0 \text{ for all } i\}.
\end{align*}
The canonical simplex $\Omega^\dim$ living in $\RR^\dim$ is given by
\begin{align*}
\{(x_0, \dotsc, x_{\dim-1}) \in \RR^{\dim} &\suchthat x_0 + \dotsb + x_{\dim-1} \leq 1 \optionalbreak
\text{ and } x_i \geq 0 \text{ for all } i\}.
\end{align*}
Note that $\Delta^\dim$ is the facet of $\Omega^{\dim+1}$ opposite to the origin.

Let $B_n$ denote the $n$-dimensional Euclidean ball.


The complete homogeneous symmetric polynomial of degree $d$ in variables $u_0, \ldots, u_{\dim}$, denoted $h_n(u_0, \ldots, u_{\dim})$,
is the sum of all monomials of degree $d$ in the variables:
\begin{align*}
h_d(u_0, \ldots, u_\dim)
&= \sum_{k_0 + \dotsb + k_\dim = d} u_0^{k_0} \dotsm u_\dim^{k_\dim} \optionalbreak
= \sum_{0 \leq i_0 \leq i_1 \leq \dotsb \leq i_d \leq \dim} u_{i_0} u_{i_1} \dotsm u_{i_d}.
\end{align*}

Also define the $d$-th power sum as
$$
p_d(u_0, \ldots, u_\dim) = u_0^d + \ldots + u_\dim^d.
$$

For a vector $u = (u_0, u_1, \ldots, u_\dim)$, we define
\[
u^{(2)} = (u_0^2, u_1^2, \ldots, u_\dim^2).
\]
Vector $\ones$ denotes the all ones vector (the dimension of the vector will be clear from the context).

A random vector $X \in \RR^\dim$ is \emph{isotropic} if $\e(X)=0$ and $\e(X X^T) = I$. A compact set in $\RR^\dim$ is isotropic if a uniformly distributed random vector in it is isotropic. The inradius of an isotropic $n$-simplex is $\sqrt{(n+2)/n}$, the circumradius is $\sqrt{n(n+2)}$.

The total variation distance between two probability measures is $d_{TV}(\mu, \nu) = \sup_A \abs{\mu(A) - \nu(A)}$ for measurable $A$. For two compact sets $K, L \subseteq \RR^\dim$, we define the total variation distance $d_{TV}(K, L)$ as the total variation distance between the corresponding uniform distributions on each set. It can be expressed as
\[
    d_{TV}(K,L) =
    \begin{cases}
        \frac{\vol{K \setminus L}}{\vol K} & \text{if $\vol K \geq \vol L$}, \\
        \frac{\vol{L \setminus K}}{\vol L} & \text{if $\vol L > \vol K$.}
    \end{cases}
\]
This identity implies the following elementary estimate:
\begin{lemma}\label{lem:bmtv}
Let $K, L$ be two compact sets in $\RR^\dim$. Let $0 < \alpha \leq 1 \leq \beta$ such that
$
\alpha K \subseteq L \subseteq \beta K$.
Then $d_{TV}(K,L) \leq 2\left(1-(\alpha/\beta)^\dim\right)$.
\end{lemma}
\begin{proof}
We have $d_{TV}(\alpha K, \beta K) = 1-(\alpha/\beta)^\dim$. Triangle inequality implies the desired inequality.
\end{proof}

\begin{lemma}\label{lem:coupons}
Consider the coupon collector's problem with $n$ coupons where every coupon occurs with probability at least $\alpha$. Let $\delta >0$. Then with probability at least $1-\delta$ all coupons are collected after $\alpha^{-1} (\log n + \log 1/\delta)$ trials.
\end{lemma}
\begin{proof}
The probability that a particular coupon is not collected after that many trials is at most
\[
(1-\alpha)^{\alpha^{-1} (\log n + \log 1/\delta)} \leq e^{-\log n - \log 1/\delta} = \delta/n.
\]
The union bound over all coupons implies the claim.
\end{proof}

For a point $x \in \RR^{\dimension}$, $\norm{x}_p = (\sum_{i = 1}^{\dimension} |x_i|^p)^{1/p}$ is the standard $\ell_p$ norm.  The unit $\ell_p^n$ ball is defined by 
\[
B_p^{\dimension} = \{x \in \RR^{\dimension}: \norm{x}_p \leq 1 \}.
\]

The Gamma distribution is denoted as $\mathrm{Gamma}(\alpha, \beta)$ and has density $f(x; \alpha, \beta) = \frac{\beta^\alpha}{\Gamma(\alpha)} x^{\alpha - 1} e^{- \beta x } 1_{x \geq 0}$, 
with shape parameters $\alpha, \beta >0$.
$\mathrm{Gamma}(1, \lambda)$ is the exponential distribution, denoted $\expdist{\lambda}$. 
The Gamma distribution also satisfies the following additivity property: If $X$ is distributed as $\mathrm{Gamma}(\alpha, \beta)$ and $Y$ is distributed as $\mathrm{Gamma}(\alpha', \beta)$, then $X+Y$ is distributed as $\mathrm{Gamma}(\alpha + \alpha', \beta)$.

The cone measure on the surface $\partial K$ of centrally symmetric convex body $K$ in $\RR^{\dimension}$ \cite{barthe2005probabilistic, schechtman1990volume, rachev1991approximate, MR1396997} is 
defined by
\[
\mu_{K}(A) = \frac{\vol(ta; a \in A, 0 \leq t \leq 1)}{\vol(K)}.
\] 
It is easy to see that $\mu_{B_p^{\dimension}}$ is uniform on $\partial B_p^{\dimension}$ for $p \in \{1, 2, \infty\}$.

From \cite{schechtman1990volume} and \cite{rachev1991approximate} we have the following representation of the cone measure on $\partial B_p^n$:
\begin{theorem}\label{thm:cone}
Let $G_1, G_2, \dotsc, G_n$ be iid random variables with density proportional to $\exp(-\abs{t}^p)$. Then the random vector $X = G/\norm{G}_p$
is independent of $\norm{G}_p$. Moreover, $X$ is distributed according to $\mu_{B_p^{\dimension}}$.
\end{theorem}

From \cite{barthe2005probabilistic}, we also have the following variation, a representation of the uniform distribution in $B_p^n$:
\begin{theorem} \label{theorem:fullnaor} 
Let $G = (G_1, \dotsc, G_{\dimension})$ be iid random variables with density proportional to $\exp(-|t|^p)$. Let $Z$ be a random variable distributed as $\expdist{1}$, independent of $G$. Then the random vector
\[
V = \frac{G}{\bigl( \sum_{i=1}^{\dimension} \abs{G_i}^p + Z\bigr)^{1/p}}
\]
is uniformly distributed in $B_p^n$.
 \end{theorem} 

 See e.g. \cite[Section 20]{MR1324786} for the change of variable formula in probability.

\section{Computing the moments of a simplex}\label{sec:moment}
The $k$-th moment $m_k(u)$ over $\Delta^\dim$ is the function
\[
u \mapsto \mathbb{E}_{X \in \Delta_\dim}((u\cdot X)^k).
\]
In this section we present a formula for the moment over $\Delta^\dim$.
Similar more general formulas appear in \cite{LasserreA}.
We will use the following result from \cite{GrundmannM78} for $\alpha_i\geq 0$:
\[
\int_{\Omega^{\dim+1}} x_0^{\alpha_0} \dotsm x_\dim^{\alpha_\dim} dx = \frac{\alpha_0! \dotsm \alpha_\dim!}{(\dim+1+ \sum_i \alpha_i)!}.
\]

From the above we can easily derive a formula for integration over $\Delta^\dim$:
$$
\int_{\Delta^\dim} x_0^{\alpha_0} \dotsm x_\dim^{\alpha_\dim} dx = \sqrt{n+1} \cdot \frac{\alpha_0! \dotsm \alpha_\dim!}{(\dim + \sum_i \alpha_i)!}.
$$

Now
\begin{align*}
\int_{\Delta^\dim} &(x_0 u_0 + \ldots + x_\dim u_\dim)^k dx \\
&= \sum_{k_0 + \dotsb + k_\dim = k} \binom{k}{k_0!, \ldots, k_\dim!} u_0^{k_0} \ldots u_\dim^{k_\dim} \int_{\Delta^\dim} x_0^{k_0} \ldots x_\dim^{k_\dim} dx \\
&= \sum_{k_0 + \dotsb + k_\dim = k} \binom{k}{k_0!, \ldots, k_\dim!} u_0^{k_0} u_0^{k_0} \ldots u_\dim^{k_\dim}
\frac{ \sqrt{\dim+1} \cdot k_0! \ldots k_\dim!}{(\dim + \sum_i k_i)!} \\
&= \frac{k! \sqrt{\dim+1}}{(\dim+k)!}  \sum_{k_0 + \dotsb + k_\dim = k} u_0^{k_0} \ldots u_\dim^{k_\dim} \\
&= \frac{k! \sqrt{\dim+1}}{(\dim+k)!} \; h_k(u).
\end{align*}

The variant of Newton's identities for the complete homogeneous symmetric polynomial gives the following relations which
can also be verified easily by direct computation:
$$ 3 h_3(u) = h_2(u) p_1(u) + h_1(u) p_2(u) + p_3(u),$$
$$ 2 h_2(u) = h_1(u) p_1(u) + p_2(u) = p_1(u)^2 + p_2(u).$$

Divide the above integral by the volume of the standard simplex $|\Delta_n|=\sqrt{\dim+1}/{\dim!}$
to get the moment:
\begin{eqnarray*}
m_3(u) & = & \frac{3! \sqrt{\dim+1}}{(\dim+3)!} h_3(u) / |\Delta_n| \\
       & = & \frac{2 (h_2(u) p_1(u) + h_1(u) p_2(u) + p_3(u))}{(\dim+1)(\dim+2)(\dim+3)}  \\
       & = & \frac{(p_1(u)^3 + 3 p_1(u) p_2(u) + 2 p_3(u))}{(\dim+1)(\dim+2)(\dim+3)} .
\end{eqnarray*}

\section{Subroutine for finding the vertices of a rotated standard simplex} \label{sec:standard}
In this section we solve the following simpler problem: Suppose we have $\text{poly}(\dim)$ samples from a rotated copy
$S$ of the standard simplex, where the rotation is such that it leaves $\ones$ invariant. The problem
is to approximately estimate the vertices of the rotated simplex from the samples.

We will analyze our algorithm in the coordinate system in which the input simplex is the standard simplex. This is only
for convenience in the analysis and the algorithm itself does not know this coordinate system.


As we noted in the introduction, our algorithm is inspired by the algorithm of \noname{Nguyen and Regev~}\cite{NguyenR09} for
the related problem of learning hypercubes and also by the FastICA algorithm in \cite{Hyvarinen99}. New ideas are
needed for our algorithm for learning simplices; in particular, our update rule is different. With the right update
rule in hand the analysis turns out to be quite similar to the one in \cite{NguyenR09}.

We want to find local maxima of the sample third moment. A natural approach to do this would be to use gradient
descent or Newton's method (this was done in \cite{FJK}). Our algorithm, which only uses first order information, can
be thought of as a fixed point algorithm leading to a particularly simple analysis and fast convergence. Before stating
our algorithm we describe the update rule we use.

We will use the abbreviation $C_\dim = (\dim+1)(\dim+2)(\dim+3)/6$. Then, from the expression for $m_3(u)$ we get
\[
\nabla m_3(u) =  \frac{1}{6C_\dim} \left(3 p_1(u)^2 \ones + 3 p_2(u) \ones + 6 p_1(u) u + 6 u^{(2)}\right).
\]
Solving for $u^{(2)}$ we get
\begin{align}
u^{(2)} &= C_\dim \nabla m_3(u) - \frac{1}{2} p_1(u)^2 \ones - \frac{1}{2} p_2(u) \ones -  p_1(u) u \nonumber \\
       &=  C_\dim  \nabla m_3(u) - \frac{1}{2} (u \cdot \ones)^2 \ones - \frac{1}{2} (u \cdot u)^2 \ones - (u \cdot \ones) u.\label{equ:squaring}
\end{align}

While the above expressions are in the coordinate system where the input simplex is the canonical simplex, the
important point is that all terms in
the last expression can be computed in any coordinate system that is obtained by a rotation leaving $\ones$ invariant.
Thus, we can compute $u^{(2)}$ as well independently of what coordinate system we are working in. This immediately
gives us the algorithm below.
We denote by $\hat{m}_3(u)$ the sample third moment, i.e., $\hat{m}_3(u)=\frac{1}{t}\sum_{i=1}^{t}(u \cdot r_i)^3$ for $t$ samples.
This is a polynomial in $u$, and the gradient is computed in the obvious way. Moreover, the gradient of the sample moment is clearly an unbiased estimator of the gradient of the moment; a bound on the deviation is given in the analysis (Lemma \ref{lem:gradient}).
For each evaluation of the gradient of the sample moment, we use a fresh sample.

It may seem a bit alarming that the fixed point-like iteration is squaring the coordinates of $u$, leading to an extremely fast growth (see Equation \ref{equ:squaring} and Subroutine 1). But, as in other algorithms having quadratic convergence like certain versions of Newton's method, the convergence is very fast and the number of iterations is small. We show below that it is $O(\log(n/\delta))$, leading to a growth of $u$ that is polynomial in $n$ and $1/\delta$. The boosting argument described in the introduction makes the final overall dependence in $\delta$ to be only linear in
$\log (1/\delta)$.

We state the following subroutine for $\mathbb{R}^n$ instead of $\mathbb{R}^{n+1}$ (thus it is learning a rotated copy of
$\Delta^{n-1}$ instead of $\Delta^n$). This is for notational convenience so that we work with $n$ instead of $n+1$.

\begin{subroutine}
\caption{Find one vertex of a rotation of the standard simplex $\Delta^{n-1}$ via a fixed point iteration-like algorithm}\label{alg:vertex}
\begin{algorithmic}
\State Input: Samples from a rotated copy of the $n$-dimensional standard simplex (for a rotation that leaves $\ones$ invariant).
\State Output: An approximation to a uniformly random vertex of the input simplex.
\vspace{.2em}
\hrule
\vspace{.2em}
\State  Pick $u(1) \in S^{\dim-1}$, uniformly at random.
\myFor{$i=1$ to $r$}
\begin{align*}
u(i+1) := &C_{\dim-1} \nabla \hat{m}_3(u(i)) - \frac{1}{2} (u(i) \cdot \ones)^2 \ones \optionalbreak
- \frac{1}{2} (u(i) \cdot u(i))^2 \ones - (u(i) \cdot \ones) u(i).
\end{align*}
\State Normalize $u(i+1)$ by dividing by $\norm{u(i+1)}_2$.
\myEndFor
\State Output $u(r+1)$.
\end{algorithmic}
\end{subroutine}

\begin{lemma}\label{lem:gradient}
Let $c>0$ be a constant, $\dim > 20$, and $0<\delta<1$. Suppose that Subroutine 1 uses a sample of size
$t = 2^{17}\dim^{2c+22}(\frac{1}{\delta})^2\ln\frac{2\dim^5 r}{\delta}$ for each evaluation of the gradient and runs for
$r =\log \frac{4(c+3) n^2\ln{n}}{\delta}$
iterations. Then with probability at least $1-\delta$
Subroutine 1 outputs a vector within distance $1/n^c$ from a vertex
of the input simplex.  With respect of the process of picking a sample and running the algorithm, each vertex is equally likely to be the nearest.

\end{lemma}

Note that if we condition on the sample, different vertices are not equally likely over the randomness of the algorithm. That is, if we try to find all vertices running the algorithm multiple times on a fixed sample, different vertices will be found with different likelihoods.
\begin{proof}
Our analysis has the same outline as that of \noname{Nguyen and Regev~}\cite{NguyenR09}.
This is because the iteration that we get is the same as that of \cite{NguyenR09} except that cubing is replaced by
squaring (see below); however some details in our proof are different.
In the proof below, several of the inequalities are quite loose and are so chosen to make the computations
simpler.

We first prove the lemma assuming that the gradient computations are exact and then show how to handle samples.
We will carry out the analysis in the coordinate system where the given simplex is the standard simplex. This is
only for the purpose of the analysis, and this coordinate system is not known to the algorithm.
Clearly, $u(i+1) =  (u(i)^2_1, \ldots, u(i)_{\dim}^2)$.
It follows that, $$u(i+1) = (u(1)_1^{2^i}, \ldots, u(1)_{\dim}^{2^i}).$$
Now, since we choose $u(1)$ randomly, with probability at least $(1-(n^2-n)\delta')$
 one of the
coordinates of $u(1)$ is greater than all the other coordinates in absolute value by a factor of at least $(1+ \delta')$,
where $0 < \delta' < 1$.
(A similar argument is made in \cite{NguyenR09} with different parameters.
 We briefly indicate the proof for our case:
The probability that the event in question does not happen is less than the probability that there are two coordinates $u(1)_a$
and $u(1)_b$ such that their absolute values are within factor $1+\delta'$, i.e.
$1/(1+\delta') \leq |u(1)_a|/|u(1)_b| < 1+\delta'$. The probability that for given $a, b$ this event happens can be seen as the
Gaussian area of the four sectors (corresponding to the four choices of signs of $u(1)_a, u(1)_b$) in the plane each with angle
less than $2\delta'$. By symmetry, the Gaussian volume of these sectors is $2\delta'/(\pi/2) < 2\delta'$.
The probability that
such a pair $(a,b)$ exists is less than $2 \binom{n}{2} \delta'$.)
Assuming this happens, then after $r$ iterations, the ratio between the largest coordinate (in absolute value) and the absolute value of any other coordinate is at least
$(1+\delta')^{2^r}$.
Thus, one of the coordinates is very
close to $1$ and others are very close to $0$, and so $u(r+1)$ is very close to a vertex of the input simplex.

Now we drop the assumption that the gradient is known exactly. For each evaluation of the gradient we use
a fresh subset of samples of $t$ points. Here $t$ is chosen so that each evaluation of the gradient is within $\ell_2$-distance
$1/\dim^{c_1}$ from its true value with probability at least $1-\delta''$, where $c_1$ will be set at the end of the proof.
An application of the Chernoff bound yields that we can take
$t = 200\dim^{2c_1+4}\ln\frac{2\dim^3}{\delta''}$; we omit the details. Thus all the $r$ evaluations of the gradient are
within distance $1/\dim^{c_1}$ from their true values with probability at least $1-r\delta''$.

We assumed that our starting vector $u(1)$ has a coordinate greater than every other coordinate by a factor
of
$(1+\delta')$ in absolute value; let us assume without loss of generality that this is the first coordinate.
Hence $|u(1)_1| \geq 1/\sqrt{n}$.
When expressing $u^{(2)}$ in terms of the gradient, the gradient gets multiplied by $C_{n-1} < n^3$ (we are assuming $n>20$),
keeping this in mind and letting $c_2 = c_1-3$ we get for $j \neq 1$
\begin{align*}
\frac{|u(i+1)_1|}{|u(i+1)_j|} \geq \frac{u(i)_1^2-1/n^{c_2}}{u(i)_j^2+1/n^{c_2}} \geq \frac{u(i)_1^2 (1-n^{-(c_2-1)})}{u(i)_j^2+1/n^{c_2}}.
\end{align*}

If $u(i)_j^2 > 1/\dim^{c_2-c_3}$, where $1\leq c_3 \leq c_2 -2$ will be determined later, then we get
\begin{align} \label{eq:if}
|u(i+1)_1|/|u(i+1)_j|
&> \frac{1-1/n^{c_2-1}}{1+1/n^{c_3}} \cdot \left(\frac{u(i)_1}{u(i)_j}\right)^2 \nonumber\\
&>  (1-1/n^{c_3})^2 \left(\frac{u(i)_1}{u(i)_j}\right)^2.
\end{align}

Else,
\begin{align*}
|u(i+1)_1|/|u(i+1)_j|
&> \frac{1/n-1/n^{c_2}}{1/n^{c_2-c_3}+1/n^{c_2}} \\
&> \left(1-\frac{1}{n^{c_3}}\right)^2 \cdot n^{c_2-c_3 -1} \\
&> \frac{1}{2} n^{c_2-c_3-1},
\end{align*}
where we used $c_3 \geq 1$ and $n>20$ in the last inequality.


We choose $c_3$ so that
\begin{align}\label{eq:c3}
\left(1-\frac{1}{n^{c_3}}\right)^2(1+\delta') > (1+\delta'/2).
\end{align}
For this, $\delta' \geq 32/n^{c_3}$ or equivalently $c_3 \geq (\ln{(32/\delta')})/\ln{n}$ suffices.

For $c_3$ satisfying \eqref{eq:c3} we have $(1-\frac{1}{n^{c_3}})^2(1+\delta')^2 > (1+\delta')$.
It then follows from \eqref{eq:if} that
the first coordinate continues to remain the largest in absolute value by a factor of at least $(1+\delta')$ after each iteration.
Also, once we have $|u(i)_1|/|u(i)_j| > \frac{1}{2} n^{c_2-c_3-1}$, we
have  $|u(i')_1|/|u(i')_j| > \frac{1}{2} n^{c_2-c_3-1}$ for all $i'>i$.

\eqref{eq:if} gives that after $r$ iterations we have
\begin{align*}
\frac{|u(r+1)_1|}{|u(r+1)_j|}
&>  (1-1/n^{c_3})^{2+2^2+\ldots+2^r} \left(\frac{u(1)_1}{u(1)_j}\right)^{2^r} \\
&\geq (1-1/n^{c_3})^{2^{r+1}-2} (1+\delta')^{2^r}.
\end{align*}

Now if $r$ is such that $(1-1/n^{c_3})^{2^{r+1}-2} (1+\delta')^{2^r} > \frac{1}{2} n^{c_2-c_3-1}$, we will be guaranteed that
$|u(r+1)_1|/|u(r+1)_j| > \frac{1}{2} n^{c_2-c_3-1}$. This condition is satisfied if we have
$(1-1/n^{c_3})^{2^{r+1}} (1+\delta')^{2^r} > \frac{1}{2} n^{c_2-c_3-1}$, or equivalently
$((1-1/n^{c_3})^2 (1+\delta'))^{2^r} \geq \frac{1}{2} n^{c_2-c_3-1}$.
Now using \eqref{eq:c3} it suffices to choose $r$ so that
$(1+\delta'/2)^{2^r} \geq \frac{1}{2} n^{c_2-c_3-1}$. Thus we can take $r = \log(4(c_2-c_3)(\ln{n})/\delta')$.


Hence we get $|u(r+1)_1|/|u(r+1)_j| > \frac{1}{2} n^{c_2-c_3-1}$.
It follows that for
$u(r+1)$, the $\ell_2$-distance from the vertex $(1, 0, \ldots, 0)$ is at most $8/n^{c_2-c_3-2} < 1/n^{c_2-c_3-3}$ for $\dim > 20$; we omit
easy details.

Now we set our parameters: $c_3 = 1+(\ln(32/\delta')/\ln{n})$ and $c_2 - c_3 - 3 = c$ and
$c_1 = c_2 + 3 = 7 + c + \ln(32/\delta')/\ln{n}$ satisfies all the constraints we imposed on $c_1, c_2, c_3$.
Choosing $\delta'' = \delta'/r$, we get that the procedure succeeds with probability at least
$1-(\dim^2-\dim)\delta' - r\delta'' > 1-\dim^2\delta'$. Now setting $\delta'=\delta/n^2$ gives the overall probability of error $\delta$,
and the number of samples and iterations as claimed in the lemma.
\end{proof}

\section{Learning simplices} \label{sec:general}

In this section we give our algorithm for learning general simplices, which uses the subroutine from the previous section.
The learning algorithm uses an affine map $T:\RR^{\dim} \to \RR^{\dim+1}$ that maps some isotropic simplex to the standard simplex.
We describe now a way of constructing such a map: Let $A$ be a matrix having as columns an orthonormal basis of $\ones^\perp$ in $\RR^{\dim+1}$.
To compute one such $A$, one can start with the $(n+1)$-by-$(n+1)$ matrix $B$ that has ones in the diagonal and first column, everything else is zero.
Let $QR=B$ be a QR-decomposition of $B$. By definition we have that the first column of $Q$ is parallel to $\ones$ and the rest of the columns span $\ones^\perp$.
Given this, let $A$ be the matrix formed by all columns of $Q$ except the first.
We have that the set $\{A^T e_i\}$ is the set of vertices of a regular $n$-simplex.
Each vertex is at distance
\[
\sqrt{\left(1-\frac{1}{n+1}\right)^2 + \frac{n}{(n+1)^2}} = \sqrt{\frac{n}{n+1}}
\]
from the origin, while an isotropic simplex has vertices at distance $\sqrt{n(n+2)}$ from the origin. So an affine transformation that maps an isotropic simplex in $\RR^\dim$ to the standard simplex in $\RR^{\dim+1}$ is $T(x) = \frac{1}{\sqrt{(n+1)(n+2)}} A x + \frac{1}{n+1} \ones_{\dim+1}$.


\begin{algorithm}
\caption{Learning a simplex.}\label{alg:simplex}
\begin{algorithmic}
\State Input: Error parameter $\eps>0$. Probability of failure parameter $\delta>0$. Oracle access to random points from some $n$-dimensional simplex $S_{INPUT}$.

\State Output: $V =\{v(1), \dotsc, v(\dim+1)\} \subseteq \RR^{\dim}$ (approximations to the vertices of the simplex).
\vspace{.2em}
\hrule
\vspace{.2em}
\State Estimate the mean and covariance using $t_1=\poly(n,1/\eps, 1/\delta)$ samples $p(1), \dotsc, p(t_1)$:
\[
\mu = \frac{1}{t_1} \sum_i p(i),
\]
\[
\Sigma = \frac{1}{t_1} \sum_i (p(i)-\mu) (p(i)-\mu)^T.
\]
\State Compute a matrix $B$ so that $\Sigma = BB^T$ (say, Cholesky decomposition).

\State Let $U = \emptyset$.
\myFor{$i=1$ to $m$ (with $m = \poly (n, \log 1/\delta))$}

\State Get $t_3=\poly(n, 1/\eps, \log 1/\delta)$ samples $r(1),\dotsc r(t_3)$ and use $\mu, B$ to map them to samples $s(i)$ from a nearly-isotropic simplex: $s(i) = B^{-1}(r(i)-\mu)$.

\State Embed the resulting samples in $\RR^{\dim+1}$ as a sample from an approximately rotated standard simplex:
Let $l(i) = T(s(i))$.

\State Invoke Subroutine \ref{alg:vertex} with sample $l(1), \dotsc, l(t_3)$ to get $u \in \RR^{\dim+1}$.
\State Let $\tilde u$ be the nearest point to $u$ in the affine hyperplane $\{x \suchthat x \cdot \ones = 1\}$. If $\tilde u$ is not within $1/\sqrt{2}$ of a point in $U$, add $\tilde u$ to $U$. (Here $1/\sqrt{2}$ is half of the edge length of the standard simplex.)
\myEndFor
\State Let
\begin{align*}
V &= B T^{-1}(U) + \mu \optionalbreak
= \sqrt{(n+1)(n+2)} B A^T\left(U-\frac{1}{n+1}\ones \right) + \mu.
\end{align*}
\end{algorithmic}
\end{algorithm}

To simplify the analysis, we pick a new sample $r(1), \dotsc, r(t_3)$ to find every vertex, as this makes every vertex equally likely to be found when given a sample from an isotropic simplex. (The core of the analysis is done for an isotropic simplex; this is enough as the algorithm's first step is to find an affine transformation that puts the input simplex in approximately isotropic position. The fact that this approximation is close in total variation distance implies that it is enough to analyze the algorithm for the case of exact isotropic position, the analysis carries over to the approximate case with a small loss in the probability of success. See the proof below for the details.) A practical implementation may prefer to select one such sample outside of the for loop, and find all the vertices with just that sample---an analysis of this version would involve bounding the probability that each vertex is found (given the sample, over the choice of the starting point of gradient descent) and a variation of the coupon collector's problem with coupons that are not equally likely.

\begin{proof}[\myproof of Theorem~\ref{thm:main}]
As a function of the input simplex, the distribution of the output of the algorithm is equivariant under invertible affine transformations.
Namely, if we apply an affine transformation to the input simplex, the distribution of the output is equally transformed.\footnote{To see this: the equivariance of the algorithm as a map between distributions is implied by the equivariance of the algorithm on any given input sample. Now, given the input sample, if we apply an affine transformation to it, this transformation is undone except possibly for a rotation by the step $s(i) = B^{-1}(r(i)-\mu)$. A rotation may remain because of the ambiguity in the characterization of $B$. But the steps of the algorithm that follow the definition of $s(i)$ are equivariant under rotation, and the ambiguous rotation will be removed at the end when $B$ is applied again in the last step.}
The notion of error, total variation distance, is also invariant under invertible affine transformations.
Therefore, it is enough to analyze the algorithm when the input simplex is in isotropic position.
In this case $\norm{p(i)} \leq n+1$ (see Section \ref{sec:preliminaries}) and we can set $t_1 \leq \poly(n,1/\eps', \log(1/\delta))$ so that $\norm{\mu} \leq \eps'$ with probability at least $1-\delta/10$ (by an easy application of Chernoff's bound), for some $\eps'$ to be fixed later.
Similarly, using results from 
\cite[Theorem 4.1]{MR2601042}, a choice of $t_1 \leq n {\eps'}^{-2} \polylog(1/\eps') \polylog(1/\delta)$ implies that the empirical second moment matrix \[\bar \Sigma = \frac{1}{t_1} \sum_i p(i) p(i)^T\] satisfies $\norm{\bar \Sigma - I} \leq \eps'$ with probability at least $1-\delta/10$. We have $\Sigma = \bar \Sigma - \mu \mu^T$ and this implies $\norm{\Sigma - I} \leq \norm{\bar \Sigma - I} + \norm{\mu \mu^T} \leq 2\eps'$.
Now, $s(1), \dotsc, s(t_3)$ is an iid sample from a simplex $S'=B^{-1}(S_{INPUT}-\mu)$. Simplex $S'$ is close in total variation distance to some isotropic simplex\footnote{The isotropic simplex $S_{ISO}$ will typically be far from the (isotropic) input simplex, because of the ambiguity up to orthogonal transformations in the characterization of $B$.} $S_{ISO}$. More precisely, Lemma~\ref{lem:tv} below shows that
\begin{equation}\label{equ:tv}
d_{TV}(S', S_{ISO}) \leq 12 \dim \eps',
\end{equation}
with probability at least $1-\delta/5$.

Assume for a moment that $s(1), \dotsc, s(t_3)$ are from $S_{ISO}$. The analysis of Subroutine \ref{alg:vertex} (fixed point-like iteration) given in Lemma~\ref{lem:gradient} would guarantee the following: Successive invocations to Subroutine~\ref{alg:vertex} find approximations to vertices of $T(S_{ISO})$ within Euclidean distance $\eps''$ for some $\eps''$ to be determined later and $t_3 = \poly(\dim,1/\eps'', \log 1/\delta)$. We ask for each invocation to succeed with probability at least $1-\delta/(20m)$ with $m = n (\log n + \log 20/\delta)$. Note that each vertex is equally likely to be found. The choice of $m$ is so that, if all $m$ invocations succeed (which happens with probability at least $1-\delta/20$), then the analysis of the coupon collector's problem, Lemma~\ref{lem:coupons}, implies that we fail to find a vertex with probability at most $\delta/20$. Overall, we find all vertices with probability at least $1-\delta/10$.

But in reality samples $s(1), \dotsc, s(t_3)$ are from $S'$, which is only \emph{close} to $S_{ISO}$. The estimate from \eqref{equ:tv} with appropriate $\eps' = \poly(1/n, \eps'', \delta)$ gives
\[
d_{TV}(S', S_{ISO}) \leq \frac{\delta}{10} \frac{1}{ t_3 m},
\]
which implies that the total variation distance between the joint distribution of all $t_3 m$ samples used in the loop and the joint distribution of actual samples from the isotropic simplex $S_{ISO}$ is at most $\delta/10$, and this implies that the loop finds approximations to all vertices of $T(S_{ISO})$ when given samples from $S'$ with probability at least $1-\delta/5$. The points in $U$ are still within Euclidean distance $
\eps''$ of corresponding vertices of $T(S_{ISO})$.

To conclude, we turn our estimate of distances between estimated and true vertices into a total variation estimate, and map it back to the input simplex. Let $S''=\conv T^{-1} U$. As $T$ maps an isotropic simplex to a standard simplex, we have that $\sqrt{(n+1)(n+2)} T$ is an isometry, and therefore the vertices of $S''$ are within distance $\eps''/\sqrt{(n+1)(n+2)}$ of the corresponding vertices of $S_{ISO}$. Thus, the corresponding support functions are uniformly within \[\eps''' = \eps''/\sqrt{(n+1)(n+2)}\] of each other on the unit sphere. This and the fact that $S_{ISO} \supseteq B_n$ imply
\[
(1 - \eps''') S_{ISO}\subseteq S'' \subseteq (1 + \eps''')S_{ISO}.
\]
Thus, by Lemma \ref{lem:bmtv}, $d_{TV}(S'', S_{ISO}) \leq
1 - (\frac{1-\eps'''}{1+\eps'''})^\dim \leq 1-(1-\eps''')^{2n}\leq 2n \eps''' \leq 2\eps''$
and this implies that the total variation distance between the uniform distributions on $\conv V$ and the input simplex is at most $2 \eps''$. Over all random choices, this happens with probability at least $1-2\delta/5$. We set $\eps'' = \eps/2$.
\end{proof}

\begin{lemma}\label{lem:tv}
Let $S_{INPUT}$ be an $\dim$-dimensional isotropic simplex. Let $\Sigma$ be an $\dim$-by-$\dim$ positive definite matrix such that $\norm{\Sigma - I} \leq \eps < 1/2$. Let $\mu$ be an $n$-dimensional vector such that $\norm{\mu} \leq \eps$. Let $B$ be an $n$-by-$n$ matrix such that $\Sigma = BB^T$. Let $S$ be the simplex $B^{-1}(S_{INPUT}-\mu)$. Then there exists an isotropic simplex $S_{ISO}$ such that $d_{TV}(S, S_{ISO}) \leq 6 \dim \eps$.
\end{lemma}
\begin{proof}
We use an argument along the lines of the orthogonal Procrustes problem (nearest orthogonal matrix to $B^{-1}$, already in \cite[Proof of Theorem 4]{NguyenR09}): Let $U D V^T$ be the singular value decomposition of $B^{-1}$. Let $R = U V^T$ be an orthogonal matrix (that approximates $B^{-1}$). Let $S_{ISO} = R S_{INPUT}$.

We have $S = UDV^T (S_{INPUT} - \mu)$. Let $\sigma_{min}$, $\sigma_{max}$ be the minimum and maximum singular values of $D$, respectively. This implies:
\begin{align}
\sigma_{min} UV^T (S_{INPUT} - \mu) &\subseteq S \subseteq \sigma_{max}UV^T (S_{INPUT} - \mu), \nonumber \\
\sigma_{min} (S_{ISO} - R\mu)  &\subseteq S \subseteq \sigma_{max}( S_{ISO} - R\mu).\label{equ:isotropy}
\end{align}
As $S_{ISO} \supseteq B_n$, $\norm{\mu} \leq 1$, $R$ is orthogonal and $S_{ISO}$ is convex, we have
\[
S_{ISO}-R\mu \supseteq (1-\norm{\mu}) S_{ISO}.
\]
Also,
\begin{align*}
S_{ISO} - R\mu &\subseteq S_{ISO} + \norm{\mu} B_n \\
&\subseteq S_{ISO} (1+ \norm{\mu}).
\end{align*}
This in \eqref{equ:isotropy} gives
\[
\sigma_{min} (1-\norm{\mu}) S_{ISO} \subseteq S \subseteq \sigma_{max}(1+\norm{\mu}) S_{ISO}.
\]
This and Lemma \ref{lem:bmtv} imply
\[
d_{TV}(S, S_{ISO}) \leq 2\left(1- \left(\frac{\sigma_{min} (1-\norm{\mu}) }{ \sigma_{max}(1+\norm{\mu})}\right)^\dim\right).
\]
The estimate on $\Sigma$ gives $\sigma_{min} \geq \sqrt{1-\eps}$, $\sigma_{max} \leq \sqrt{1+\eps}$.  Thus
\begin{align*}
d_{TV}(S, S_{ISO}) &\leq 2\left(1- \left(\frac{1-\eps}{1+\eps} \right)^{3\dim/2}\right) \\
& \leq 2\left(1- \left(1-\eps\right)^{3\dim}\right) \\
& \leq 6\dim \eps.
\end{align*}
\end{proof}

\section{The local and global maxima of the 3rd moment of the standard simplex and the isotropic simplex} \label{sec:maxima}


In this section we study the structure of the set of local maxima of the third moment as a function of the direction (which happens to be essentially $u \mapsto \sum u_i^3$ as discussed in Section \ref{sec:moment}). This is not necessary for our algorithmic result, however it gives insight into the geometry of the third moment (the location of local maxima/minima and stationary points) and suggests that more direct optimization algorithms like gradient descent and Newton's method will also work, although we will not prove that.

\begin{theorem}\label{thm:maxima}
Let $K \subseteq \RR^\dim$ be an isotropic simplex. Let $X$ be random in $K$. Let $V = \{x_i\}_{i=1}^{\dim+1} \subseteq \RR^\dim$ be the set of normalized vertices of $K$. Then $V$ is a complete set of local maxima and a complete set of global maxima of $F:S^{\dim-1} \to \RR$ given by $F(u) = \e ((u \cdot X)^3)$.
\end{theorem}
\noindent\emph{Proof idea:} Embed the simplex in $\RR^{\dim+1}$. Show that the third moment is proportional to the complete homogeneous symmetric polynomial of degree 3, which for the relevant directions is proportional to the sum of cubes. To conclude, use first and second order optimality conditions to characterize the set of local maxima.
\begin{proof}
Consider the standard simplex
\[
\Delta^n = \conv \{e_1, \dotsc, e_{\dim+1} \} \subseteq \RR^{\dim+1}
 \]
and identify it with $V$ via a linear map $A :\RR^{\dim+1} \to \RR^{\dim}$ so that $A(\Delta^n) = V$. Let $Y$ be random in $\Delta^n$. Consider $G: S^{\dim} \to \RR$ given by $G(v) = m_3(v) = \e ((v \cdot Y)^3)$. Let $U = \{ v \in \RR^{\dim+1} \suchthat v \cdot \ones =0, \norm{v}=1 \}$ be the equivalent feasible set for the embedded problem.
We have $G(v) = cF(A v)$ for any $v \in U$ and some constant $c > 0$ independent of $v$.
To get the theorem, it is enough to show that the local maxima of $G$ in $U$ are precisely the normalized versions of the projections of the canonical vectors onto the hyperplane orthogonal to $\ones = (1, \dotsc, 1)$. According to Section \ref{sec:moment}, for $v \in U$ we have
\[
G(v) \propto p_3(v).
\]
Using a more convenient but equivalent constant, we want to enumerate the local maxima of the problem
\begin{equation}\label{equ:opt}
\begin{aligned}
\max \frac{1}{3} p_3(v)& \\
\text{s.t.} \quad v \cdot v &= 1 \\
v \cdot \ones &= 0 \\
v &\in \RR^{\dim+1}.
\end{aligned}
\end{equation}
The Lagrangian function is
\[
L(v, \lambda_1, \lambda_2) = \frac{1}{3} \sum_i v_i^3 - \lambda_1 \sum_i v_i - \lambda_2\frac{1}{2} \left(\biggl(\sum_i v_i^2 \biggr) -1 \right).
\]
The first order condition is $\nabla_v L = 0$, that is,
\begin{equation}\label{equ:foc}
v_i^2 = \lambda_1 + \lambda_2 v_i \quad \text{for $i= 1, \dotsc, \dim+1$.}
\end{equation}
Consider this system of equations on $v$ for any fixed $\lambda_1$, $\lambda_2$. Let $f(x) = x^2$, $g(x)= \lambda_1 + \lambda_2 x$. The first order condition says $f(v_i) = g(v_i)$, where $f$ is convex and $g$ is affine. That is, the $v_i$s can take at most two different values. As our optimization problem \eqref{equ:opt} is symmetric under permutation of the coordinates, we conclude that, after putting the coordinates of a point $v$ in non-increasing order, if $v$ is a local maximum of \eqref{equ:opt}, then $v$ must be of the form
\[
v = (a, \dotsc, a, b, \dotsc, b),
\]
where $a>0>b$ and there are exactly $\alpha$ $a$s and $\beta$ $b$s, for $\alpha, \beta \in \{1, \dotsc, \dim\}$.

We will now study the second order necessary condition (SONC) to eliminate from the list of candidates all vectors with $\alpha >1$. It is easy to see that the surviving vectors are exactly the promised scaled projections of the canonical vectors. This vectors must all be local and global maxima: At least one of them must be a global maximum as we are maximizing a continuous function over a compact set and all of them have the same objective value so all of them are local and global maxima.

The SONC at $v$ asks for the Hessian of the Lagrangian to be negative semidefinite when restricted to the tangent space to the constraint set at $v$ \cite[Section 11.5]{MR2423726}.
We compute the Hessian (recall that $v^{(2)}$ is the vector of the squared coordinates of $v$):
\[
\nabla_v L = v^{(2)} - \lambda_1 \ones - \lambda_2 v
\]
\[
\nabla^2_v L = 2 \diag(v) - \lambda_2 I
\]
where $\diag(v)$ is the $(\dim + 1)$-by-$(\dim+1)$ matrix having the entries of $v$ in the diagonal and 0 elsewhere.

A vector in the tangent space is any $z \in \RR^{\dim+1}$ such that $z \cdot \ones = 0$, $v \cdot z = 0$, and definiteness of the Hessian is determined by the sign of $z^T \nabla^2_v L z$ for any such $z$, where
\[
z^T \nabla^2_v L z = \sum_{i=1}^{\dim+1} z_i^2 (2 v_i - \lambda_2).
\]
Suppose $v$ is a critical point with $\alpha \geq 2$. To see that such a $v$ cannot be a local maximum, it is enough to show $2a > \lambda_2$, as in that case we can take $z = (1, -1, 0,\dotsc, 0)$ to make the second derivative of $L$ positive in the direction $z$.

In terms of $\alpha, \beta, a, b$, the constraints of \eqref{equ:opt} are $\alpha a + \beta b = 0$, $\alpha a^2 + \beta b^2 = 1$, and this implies $a = \sqrt{\frac{\beta}{\alpha (\dim+1) }}$, $b = - \sqrt{\frac{\alpha}{\beta (\dim+1) }}$. The inner product between the first order condition \eqref{equ:foc} and $v$ implies $\lambda_2 = \sum v_i^3 = \alpha a^3 + \beta b^3$. It is convenient to consider the change of variable $\gamma = \alpha/(\dim+1)$, as now candidate critical points are parameterized by certain discrete values of $\gamma$ in $(0,1)$. This gives $\beta = (1-\gamma)(\dim+1)$, $ a = \sqrt{(1-\gamma)/(\gamma(\dim+1))}$ and
\begin{align*}
\lambda_2 &= (\dim+1)\biggl[\gamma \left(\frac{1-\gamma}{\gamma (\dim+1)}\right)^{3/2} \\
&\qquad - (1-\gamma)\left(\frac{\gamma}{(1-\gamma) (\dim+1)}\right)^{3/2}\biggr] \\
&= \frac{1}{\sqrt{(\dim+1) \gamma (1-\gamma)}} \left[(1-\gamma)^2 - \gamma^2\right] \\
&= \frac{1}{\sqrt{(\dim+1) \gamma (1-\gamma)}} [1- 2\gamma].
\end{align*}
This implies
\begin{align*}
2 a - \lambda_2 &= \frac{1}{\sqrt{(\dim+1) \gamma (1-\gamma)}} [2(1-\gamma) -1 + 2\gamma] \\
    &= \frac{1}{\sqrt{(\dim+1) \gamma (1-\gamma)}}.
\end{align*}
In $(0,1)$, the function given by $\gamma \mapsto 2a-\lambda_2 = \frac{1}{\sqrt{(\dim+1) \gamma (1-\gamma)}}$ is convex and symmetric around $1/2$, where it attains its global minimum value, $2/ \sqrt{\dim+1}$, which is positive.
\end{proof}

\section{Probabilistic Results} \label{sec:prob} 

In this section we show the probabilistic results underlying the reductions from learning simplices and $\ell_p^n$ balls to ICA. The results are Theorems \ref{thm:simplexscaling} and \ref{thm:lpscaling}. They each show a simple non-linear rescaling of the respective uniform distributions that gives a distribution with independent components (Definition \ref{def:ic}).

Theorem \ref{thm:simplexscaling} below gives us, in a sense, a ``reversal'' of the representation of the cone measure on $\partial B_p^n$, seen in Theorem \ref{thm:cone}. Given any random point in the standard simplex, we can apply a simple non-linear scaling and recover a distribution with independent components.  

\begin{definition}\label{def:ic}
We say that a random vector $X$ has \emph{independent components} if it is an affine transformation of a random vector having independent coordinates.
\end{definition}

\begin{theorem}\label{thm:simplexscaling} 
Let $X$ be a uniformly random vector in the $(\dimension-1)$-dimensional standard simplex $\Delta_{n-1}$. Let $T$ be a random scalar distributed as $\mathrm{Gamma}(n, 1)$. Then the coordinates of $T X$ are iid as $\expdist{1}$.

Moreover, if $A:\mathbb{R}^{\dimension} \rightarrow \mathbb{R}^{\dimension}$ is an invertible linear transformation, then the random vector $TA(X)$ has independent components.
\end{theorem}
\begin{proof}
In the case where $p = 1$, Theorem \ref{thm:cone} restricted to the positive orthant implies that for random vector $G = (G_1, \dotsc, G_{\dimension})$, if each $G_i$ is an iid exponential random variable $\expdist{1}$, then $( G/\norm{G}_1, \norm{G}_1)$ has the same (joint) distribution as $(X,T)$. Given the measurable function $f(x,t) = xt$, $f(X,T)$ has the same distribution as $f( G/\norm{G}_1, \norm{G}_1)$. That is, $XT$ and $G$ have the same distribution\footnote{See \cite[Theorem 1.1]{MR1456629} for a similar argument in this context.}.

For the second part, we know $T A(X) = A(TX)$ by linearity. By the previous argument the coordinates of $TX$ are independent. This implies that $A(TX)$ has independent components.
\end{proof}

The next lemma complements the main result in \cite{barthe2005probabilistic}, Theorem 1 (Theorem \ref{theorem:fullnaor} elsewhere here). They show a representation of the uniform distribution in $B_p^n$, but they do not state the independence that we need for our reduction to ICA.
\begin{lemma} \label{lemma:independence} 
Let $p \in [1, \infty)$. Let $G=(G_1, \dotsc, G_n)$ be iid random variables with density proportional to $\exp(-\abs{t}^p)$. Let $W$ be a nonnegative random variable with distribution $\expdist{1}$ and independent of $G$. Then the random vector
\[
\frac{G}{(\norm{G}_p^p + W)^{1/p}}
\]
is independent of $(\norm{G}_p^p + W)^{1/p}$.
\end{lemma}

\begin{proofidea}
We aim to compute the join density, showing that it is a product of individual densities. To avoid complication, we raise everything to the $p$th power, which eliminates extensive use of the chain rule involved in the change of variables. 
\end{proofidea}

\begin{proof}
It is enough to show the claim conditioning on the orthant in which $G$ falls, and by symmetry it is enough to prove it for the positive orthant. Let random variable $H = (G_1^p, G_2^p, \dotsc, G_n^p)$.
Since raising (strictly) positive numbers to the $p$th power is injective, it suffices to show that the random vector 
\[
X = \frac{H}{\sum_{i=1}^n{H}_i + W}
\]
is independent of the random vector $Y = \sum_{i=1}^n{H}_i + W$. 
 
First, let $U$ be the interior of the support of $(X,Y)$, that is $U = \{ x \in \RR^{n} : x_i > 0, \sum_i x_i < 1 \}\times \{y \in \RR :y>0\}$  and consider $h: U \rightarrow \RR^{\dimension}$ and $w: U \rightarrow \RR$ where
\[
h(x,y) = x y
\] and 
\[
w(x,y) = y - \sum_{i=1}^n h(x,y)_i = y - \sum_{i=1}^n x_i \cdot y = y\left(1 - \sum_{i=1}^n x_i\right).
\]
The random vector $(H,W)$ has a density $f_{H,W}$ supported on $V  = \operatorname{int} \RR^{n+1}_+$ and 
$$(x,y) \mapsto (h(x,y), w(x,y))$$
is one-to-one from $U$ onto $V$. 
Let $J(x,y)$ be the determinant of its Jacobian. This Jacobian is
\[
\begin{pmatrix}
y & 0 & \cdots & 0 & x_1 \\
0 & y & \cdots & 0 & x_2 \\
\vdots & & & & \vdots \\
0 & 0 & \cdots & y & x_n \\
-y & -y & \cdots & -y & 1 - \sum_{i=1}^n x_i \\
\end{pmatrix}
\]
which, by adding each of the first $n$ rows to the last row, reduces to
\[
\begin{pmatrix}
y & 0 & \cdots & 0 & x_1 \\
0 & y & \cdots & 0 & x_2 \\
\vdots & & & & \vdots \\
0 & 0 & \cdots & y & x_n \\
0 & 0 & \cdots & 0 & 1 \\
\end{pmatrix},
\]
the determinant of which is trivially $J(x,y) = y^n$.

We have that $J(x,y)$ is nonzero in $U$. Thus, $(X,Y)$ has density $f_{X,Y}$ supported on $U$ given by
\begin{align*}
f_{X,Y}(x,y) &= f_{H,W}\bigl(h(x,y),w(x,y)\bigr) \cdot \abs{J(x)}.
\end{align*}

It is easy to see\footnote{See for example \cite[proof of Theorem 3]{barthe2005probabilistic}.} that each $H_i = G_i^p$ has density $\mathrm{Gamma}(1/p, 1)$ and thus $\sum_{i=1}^n H_i$ has density $\mathrm{Gamma}(n/p,1)$ by the additivity of the Gamma distribution.  We then compute the joint density
\begin{align*}
f_{X,Y}(x,y) &= f_{H,W}\bigl(h(x,y),w(x,y)\bigr) \cdot y^n \\
&= f_{H,W}\Bigl(xy, y(1 - \sum\limits_{i=1}^n x_i)\Bigr) \cdot y^n.
\end{align*}
Since $W$ is independent of $H$,
\begin{align*}
f_{X,Y}(x,y) &=  f_{W}\bigg(y\Big(1 - \sum\limits_{i=1}^n x_i\Big)\bigg) \cdot y^n \prod_{i=1}^n f_{H_i}(x_iy) 
\end{align*}
where
\begin{align*}
\prod_{i=1}^{\dimension}\Big(f_{H_i}(x_iy)\Big) \cdot f_{W}\bigg(y\Big(1 - \sum\limits_{i=1}^n x_i\Big)\bigg) \cdot y^n 
&\propto \prod_{i=1}^n \left[e^{-x_iy}(x_iy)^{\frac{1}{p} - 1}\right] \exp\left(-y(1-\sum\limits_{i=1}^n x_i)\right) y^n \\
&\propto \biggl(\prod_{i=1}^n x_i^{\frac{1}{p} - 1}\biggr) y^{n/p}.
\end{align*}
The result follows.
\end{proof}

With this in mind, we show now our analog of Theorem \ref{thm:simplexscaling} for $B_p^{\dimension}$.

\begin{theorem}\label{thm:lpscaling}
Let $X$ be a uniformly random vector in $B_p^{\dimension}$. Let $T$ be a random scalar distributed as $\mathrm{Gamma}((n/p)+1,1)$. Then the coordinates of $T^{1/p} X$ are iid, each with density proportional to $\exp(-\abs{t}^p)$.
Moreover, if $A: \RR^{\dimension} \rightarrow \RR^{\dimension}$ is an invertible linear transformation, then the random vector given by $T^{1/p}A(X)$ has independent components.
\end{theorem}

\begin{proof}
Let $G = (G_1, \dotsc, G_{\dimension})$ where each $G_i$ is  iid as $\mathrm{Gamma}(1/p, 1)$. Also, let $W$ be an independent random variable distributed as $\expdist{1}$. Let $S = \bigl(\sum_{i=1}^n |G_i|^p + W \bigr)^{1/p}$. 

By Lemma \ref{lemma:independence} and Theorem \ref{theorem:fullnaor} we know $(G/S,S)$ has the same joint distribution as $(X,T^{1/p})$. Then for the measurable function $f(x,t) = xt$, we immediately have $f(X,T^{1/p})$ has the same distribution as $f(G/S, S)$ and thus $XT^{1/p}$ has the same distribution as $G$. 

For the second part, since $T$ is a scalar, we have $T^{1/p}A(X) = A(T^{1/p}X)$. By the previous argument we have that the coordinates of $T^{1/p}X$ are independent. Thus, $A(T^{1/p}X)$ has independent components. 
\end{proof}

%

This result shows that one can obtain a vector with independent components from a sample in a linearly transformed $\ell_p$ ball.
In Section \ref{sec:reduction} we show that they are related in such as way that one can recover the linear transformation from the independent components via ICA.


\section{Learning problems that reduce to ICA}\label{sec:reduction}
Independent component analysis is a certain computational problem and an associated family of algorithms. Suppose that $X$ is a random $\dimension$-dimensional vector whose coordinates are independently distributed. The coordinates' distributions are unknown and not necessarily identical. The ICA problem can be stated as follows: given samples from an affine transformation $Y=AX+b$ of $X$, estimate $A$ and $b$ (up to a certain intrinsic indeterminacy: permutation and scaling of the columns of $A$). We will state more precisely below what is expected of a an ICA algorithm.


We show randomized reductions from the following two natural statistical estimation problems to ICA:

\begin{problem}[simplex] 
Given uniformly random points from an $n$-dimensional simplex, estimate the simplex.
\end{problem} 

This is the same problem of learning a simplex as in the rest of the paper, we just restate it here for clarity.

To simplify the presentation for the second problem, we ignore the estimation of the mean of an affinely transformed distribution. That is, we assume that the $\ell_p^n$ ball to be learned has only been \emph{linearly} transformed.

\begin{problem}[linearly transformed $\ell_p^n$ balls]
Given uniformly random points from a linear transformation of the  $\ell_p^n$-ball, estimate the linear transformation.
\end{problem}

These problems do not have an obvious independence structure. Nevertheless, known representations of the uniform measure in an $\ell_p^n$ ball and the \emph{cone measure} (defined in Section \ref{sec:preliminaries}) on the surface of an $\ell_p^n$ ball can be slightly extended to map a sample from those distributions into a sample with independent components by a non-linear scaling step.
The use of a non-linear scaling step to turn a distribution into one having independent components has been done before \cite{MR2756189, DBLP:conf/nips/SinzB08}, but there it is applied \emph{after} finding a transformation that makes the distribution axis-aligned. This alignment is attempted using ICA (or variations of PCA) on the original distribution \cite{MR2756189, DBLP:conf/nips/SinzB08}, without independent components, and therefore the use of ICA is somewhat heuristic. One of the contributions of our reduction is that the rescaling we apply is ``blind'', namely, it can be applied to the original distribution. In fact, the distribution does not even need to be isotropic (``whitened''). The distribution resulting from the reduction has independent components and therefore the use of ICA on it is well justified.

The reductions are given in Algorithms \ref{alg:simplexreduction} and \ref{alg:lpreduction}. To state the reductions, we denote by $ICA(s(1), s(2), \ldots)$ the invocation of an ICA routine. 
It takes samples $s(1), s(2), \ldots$ of a random vector $Y=AX + \mu$, where the coordinates of $X$ are independent, and returns an approximation to a square matrix $M$ such that $M(Y-\e(Y))$ is isotropic and has independent coordinates. 
The theory of ICA \cite[Theorem 11]{comon1994independent} implies that if $X$ is isotropic and at most one coordinate is distributed as a Gaussian, then such an $M$ exists and it satisfies $M A = DP$, where $P$ is a permutation matrix and $D$ is a diagonal matrix with diagonal entries in $\{-1, 1\}$. We thus need the following definition to state our reduction: Let $c_{p,n} = (\e_{X \in B_p^n}(X_1^2))^{1/2}$. That is, the uniform distribution in $B_p^n/c_{p,n}$ is isotropic.

As we do not state a full analysis of any particular ICA routine, we do not state explicit approximation guarantees.

\begin{algorithm}
\caption{Reduction from Problem 1 to ICA}\label{alg:simplexreduction}

\begin{algorithmic}
\State Input: A uniformly random sample $p(1), \dotsc, p(t)$ from an $n$-dimensional simplex $S$.
\State Output: Vectors $\tilde v(1), \dotsc, \tilde v(n+1)$ such that their convex hull is close to $S$.
\vspace{.2em}
\hrule
\vspace{.2em}
\State Embed the sample in $\RR^{n+1}$: Let $p'(i) = (p(i), 1)$.

\State For every $i = 1, \dotsc, t$, generate a random scalar $T(i)$ distributed as $\mathrm{Gamma}(n+1, 1)$. Let $q(i) = p'(i) T(i)$.
\State Invoke $ICA(q(1), \dotsc, q(t))$ to obtain a approximately separating matrix $\tilde M$.
\State Compute the inverse of $\tilde M$ and multiply every column by the sign of its last entry to get a matrix $\tilde A$.
\State Remove the last row of $\tilde A$ and return the columns of the resulting matrix as $\tilde v(1), \dotsc, \tilde v(n+1)$.
\end{algorithmic}
\end{algorithm}
Algorithm \ref{alg:simplexreduction} works as follows: Let $X$ be an $(n+1)$-dimensional random vector with iid coordinates distributed as $\expdist{1}$. Let $V$ be the matrix having columns $(v(i),1)$ for $i=1, \dotsc, n+1$. Let $Y$ be random according to the distribution that results from scaling in the algorithm. Theorem \ref{thm:simplexscaling} implies that $Y$ and $VX$ have the same distribution. Also, $X-\ones$ is isotropic and $Y$ and $V(X-\ones) + V \ones$ have the same distribution. Thus, the discussion about ICA earlier in this section gives that the only separating matrices $M$ are such that $M V = D P$ where $P$ is a permutation matrix and $D$ is a diagonal matrix with diagonal entries in $\{-1, 1\}$. That is, $V P^T = M^{-1} D$. As the last row of $V$ is all ones, the sign change step in Algorithm \ref{alg:simplexreduction} undoes the effect of $D$ and recovers the correct orientation.
 
\begin{algorithm}
\caption{Reduction from Problem 2 to ICA}\label{alg:lpreduction}
\begin{algorithmic}
\State Input: A uniformly random sample $p(1), \dotsc, p(t)$ from $A(B_p^n)$ for a known parameter $p \in [1, \infty)$, where $A:\RR^n \to \RR^n$ is an unknown invertible linear transformation.
\State Output: A matrix $\tilde A$ such that $\tilde A B_p^n$ is close to $A(B_p^n)$.
\vspace{.2em}
\hrule
\vspace{.2em}

\State For every $i = 1, \dotsc, t$, generate a random scalar $T(i)$ distributed as $\mathrm{Gamma}((n/p)+1,1)$. Let $q(i) = p(i)  T(i)^{1/p}$.
\State Invoke $ICA(q(1), \dotsc, q(t))$ to obtain an approximately separating matrix $\tilde M$. 
\State Output $\tilde A = c_{p,n}^{-1} \tilde M^{-1}$.
\end{algorithmic}
\end{algorithm}
%
%

Similarly, Algorithm \ref{alg:lpreduction} works as follows: Let $X$ be a random vector with iid coordinates, each with density proportional to $\exp(-\abs{t}^p)$. Let $Y$ be random according to the distribution that results from scaling in the algorithm. Theorem \ref{thm:lpscaling} implies that $Y$ and $AX$ have the same distribution. Also, $X/c_{p,n}$ is isotropic and we have $Y$ and $Ac_{p,n}(X/c_{p,n})$ have the same distribution. Thus, the discussion about ICA earlier in this section gives that the only separating matrices $M$ are such that $M A c_{p,n} = D P$ where $P$ is a permutation matrix and $D$ is a diagonal matrix with diagonal entries in $\{-1, 1\}$. That is, $A P^{T} D^{-1} = c_{p,n}^{-1} M^{-1}$. The fact that $B_p^n$ is symmetric with respect to coordinate permutations and sign changes implies that $A P^{T} D^{-1} B_p^n = A B_p^n$ and is the same as $c_{p,n}^{-1} M^{-1}$. When $p \neq 2$, the assumptions in the discussion above about ICA are satisfied and Algorithm \ref{alg:lpreduction} is correct. When $p = 2$, the distribution of the scaled sample is Gaussian and this introduces ambiguity with respect to rotations in the definition of $M$, but this ambiguity is no problem as it is counteracted by the fact that the $l_2$ ball is symmetric with respect to rotations.




\section{Conclusion}
We showed, in two different ways, that the problem of learning simplices can be solved efficiently using techniques for ICA. We also showed that when the sample is one that may not satisfy the requirement of independent components, we can efficiently obtain from it a sample that guarantees this property and from which the original distribution can be estimated. 
Many questions remain: 
Can we do this for other polytopes? Can we do this when the points come from the Gaussian distribution with labels instead
of the uniform distribution in the polytope? In particular, does any one of the two techniques that we used 
in this paper for learning simplices extend to learning polytopes or to 
latent variable models?

\acks{We thank Santosh Vempala for telling us about the polytope learning problem, the approach of using higher-order moments and for helpful discussions.
We also thank Keith Ball, Alexander Barvinok, Franck Barthe, Mikhail Belkin, Adam Kalai, Assaf Naor, Aaditya Ramdas, Roman Vershynin and James Voss for helpful discussions.}

\bibliography{simplex_bibliography,refs}

\appendix

\end{document}